\documentclass[11pt]{article}
\usepackage[square]{natbib}
\usepackage{fullpage}

\usepackage[utf8]{inputenc}
\usepackage[T1]{fontenc}
\usepackage{mathtools}
\usepackage{bbm}
\usepackage{amsmath}
\usepackage{mathtools}
\usepackage{amsthm}
\usepackage{hyperref}
\usepackage{cleveref}
\usepackage{amssymb}
\usepackage{xcolor}
\usepackage{xspace}

\usepackage{graphicx}

\usepackage{algorithm}
\usepackage{algorithmic}
\usepackage{mdframed}

\newmdenv[linecolor=black,backgroundcolor=gray!10,linewidth=1pt]{mybox}

\newcommand{\wh}{\widehat}

\newcommand{\N}{\mathcal{N}}
\newcommand{\R}{\mathbb{R}}

\renewcommand{\hat}{\wh}

\newcommand{\rset}[1]{\{ #1 \}}

\newcommand{\uplowquant}[3]{\big(#1\big)_{#2}^{#3}}

\newcommand{\POk}{\textsc{PO}_k}
\newcommand{\etastar}{\eta^*}
\newcommand{\etabar}{\Bar{\eta}}

\newcommand{\etabarup}{\Bar{\eta}^{\text{upd}}}
\newcommand{\acks}[1]{\section*{Acknowledgments}#1}

\newcommand{\eqtriangle}[1]{\stackrel{{def}}{=}\text{#1}}
\newcommand{\ITE}[1]{\textsc{ITE}_{#1}}

\newtheorem{theorem}{Theorem}[section]
\newtheorem{lemma}[theorem]{Lemma}
\newtheorem{definition}[theorem]{Definition}

\newtheorem{proposition}[theorem]{Proposition}

\newtheorem{assumption}[theorem]{Assumption}

\newtheorem{remark}[theorem]{Remark}

\newcommand{\medianestimator}{\textsc{Median-Estimator}\xspace}
\newcommand{\variability}{\textsc{Compute-Variability}\xspace}
\newcommand{\variabilitylowerquantile}{\textsc{Compute-Variability-Lower-Quantile}\xspace}
\newcommand{\variabilityupperquantile}{\textsc{Compute-Variability-Upper-Quantile}\xspace}

\DeclareMathOperator*{\argmin}{{\text{argmin}}}
\DeclareMathOperator*{\E}{{\mathbb{E}}}

\DeclareMathOperator{\supp}{supp}

\DeclareMathOperator{\lab}{lab}
\DeclarePairedDelimiter\abs{\lvert}{\rvert}%
\usepackage{times}

\author{Raghavendra Addanki\\ Adobe Research \\ \texttt{raddanki@adobe.com} \and Siddharth Bhandari \\ Toyota Technological Institute at Chicago \\ \texttt{siddharth@ttic.edu}}

\usepackage{etoolbox}

\makeatletter

\newcommand{\define}[4][ignore]{%
  \ifstrequal{#1}{ignore}{}{
  \@namedef{thmtitle@#2}{#1}}%
  \@namedef{thm@#2}{#4}%
  \@namedef{thmtypen@#2}{lemma}%
  \newtheorem{thmtype@#2}[theorem]{#3}%
  \newtheorem*{thmtypealt@#2}{#3~\ref{#2}}%
}

\newcommand{\state}[1]{%
  \@namedef{curthm}{#1}
  \@ifundefined{thmtitle@#1}{
  \begin{thmtype@#1}
    }{
  \begin{thmtype@#1}[\@nameuse{thmtitle@#1}]
  }
    \label{#1}
    \@nameuse{thm@#1}
  \end{thmtype@#1}
  \@ifundefined{thmdone@#1}{
  \@namedef{thmdone@#1}{stated}%
  }{}
}

\newcommand{\restate}[1]{%
  \@namedef{curthm}{#1}
  \@ifundefined{thmtitle@#1}{
    \begin{thmtypealt@#1}
    }{
  \begin{thmtypealt@#1}[\@nameuse{thmtitle@#1}]
  }
    \@nameuse{thm@#1}
  \end{thmtypealt@#1}
  \@ifundefined{thmdone@#1}{
  \@namedef{thmdone@#1}{stated}%
  }{}
}

\newcommand{\thmlabel}[1]{
  \@ifundefined{thmdone@\@nameuse{curthm}}{\label{#1}
    }{\tag*{\eqref{#1}}}
}
\makeatother

\date{}
\title{Limits of Approximating the Median Treatment Effect\footnote{The author ordering is alphabetical, as is customary in the theoretical computer science community.}}

\begin{document}

\maketitle

\begin{abstract}
    Average Treatment Effect (ATE) estimation is a well-studied problem in causal inference. However, it does not necessarily capture the heterogeneity in the data, and several approaches have been proposed to tackle the issue, including estimating the Quantile Treatment Effects. In the finite population setting containing $n$ individuals, with treatment and control values denoted by the potential outcome vectors $\mathbf{a}, \mathbf{b}$, much of the prior work focused on estimating median$(\mathbf{a}) -$ median$(\mathbf{b})$, where median($\mathbf x$) denotes the median value in the sorted ordering of all the values in vector $\mathbf x$. It is known that estimating the difference of medians is easier than the desired estimand of median$(\mathbf{a-b})$, called the Median Treatment Effect (MTE). The fundamental problem of causal inference -- for every individual $i$, we can only observe one of the potential outcome values, i.e., either the value $a_i$ or $b_i$, but not both, makes estimating MTE particularly challenging. In this work, we argue that MTE is not estimable and detail a novel notion of approximation that relies on the sorted order of the values in $\mathbf{a-b}$. Next, we identify a quantity called \emph{variability} that exactly captures the complexity of MTE estimation. By drawing connections to the notions of instance-optimality studied in theoretical computer science, we show that \emph{every} algorithm for estimating the MTE obtains an approximation error that is no better than the error of an algorithm that computes variability. Finally, we provide a simple linear time algorithm for computing the variability exactly. Unlike much prior work, a particular highlight of our work is that we make no assumptions about how the potential outcome vectors are generated or how they are correlated, except that the potential outcome values are $k$-ary, i.e., take one of $k$ discrete values. 
\end{abstract}


\section{Introduction}

In the realm of experimentation, much attention is dedicated to constructing efficient estimators for the average treatment effect (ATE), where machine learning-based double robust estimation techniques are considered state-of-the-art~\citep{chernozhukov2018double, kennedy2017non}. However, making policy or downstream task decisions only based on ATE, might be sub-optimal, as there is a loss of valuable information from the heterogenous regions in the data. There have been many methods proposed in the literature for tackling this issue~\citep{imai2013estimating, athey2016recursive, ding2019decomposing}, and we focus on investigating variations across different outcome distribution quantiles.
Unlike ATE, which \textit{averages} out the responses, Quantile Treatment Effects (QTE)~\citep{chernozhukov2005iv} provide a more nuanced perspective by unveiling whether effects differ among individuals or outcomes, allowing for a more detailed understanding. This approach proves invaluable in discerning if a treatment is more impactful at specific points in the distribution, facilitating the study of policies affecting diverse population segments, and groups unfairly affected. Several cautionary tales exist in the literature~\citep{heckman1997making}, with recent examples related to Netflix's A/B testing oversight of treatment effects disparities among user engagement levels~\citep{bojinov2020avoid}.

We study the finite populating setting on $n$ individuals, where an experimenter assigns the individuals to treatment and control groups. Suppose that the treatment and control values of all the individuals be denoted by $\mathbf{a}$ and $\mathbf{b}$ respectively. Much of the work on QTE estimation has focused on estimating the difference of quantiles of treatment and control~\citep{howard2022sequential, martinez2023counterfactual}. E.g., for median estimation, the estimand studied is $\textit{median}(\mathbf{a}) - \textit{median}(\mathbf{b})$. Although efficient estimators are known, such an estimand has a limitation, as it is not based on the causal effect vector $\mathbf{a} - \mathbf{b}$, which of course is not observable. It might not necessarily capture the causal effect entirely. In contrast, some prior work considered the estimand $\textit{median}(\mathbf{a} - \mathbf{b})$, broadly speaking~\citep{heckman1997making, leqi2021median, kallus2022s, kennedy2023towards}. Recently,~\citet{kallus2022s} studied a related problem of identifying the fraction of individuals that are negatively affected (under binary potential outcomes). They provide efficient estimators, but assume the availability of covariates and make distributional assumptions. 

In this work, we study the estimation of $\textit{median}(\mathbf{a} - \mathbf{b})$, henceforth called the Median Treatment Effect (MTE) estimation, without any assumptions, except that the potential outcomes, i.e., the vectors $\mathbf{a}$ and $\mathbf{b}$ are binary (or $k$-ary). To the best of our knowledge, ours is the first work that tackles this problem in a general setting with minimal or no assumptions. We want to highlight that we \textit{do not} make any distributional assumptions, and the vectors $\mathbf{a}$ and $\mathbf{b}$ can be arbitrarily correlated. Our results contribute to the growing body of work that identifies connections between theoretical computer science and causal inference~\citep{pouget2019variance, addanki2022sample, harshaw2023balancing}. In our work, we answer the following questions.

\begin{mybox}
\begin{enumerate}
    \item  \textbf{Is Median Estimable?} We show that MTE estimation is infeasible, in stark contrast to ATE estimation. 
    In particular, we show that \emph{any} estimator (algorithm) is unable to distinguish between two distributions from which the treatment and control values are drawn if they have the same marginal distributions. 
    Further, the two distributions can have a large difference in their median values. 
    \item \textbf{Can we approximate MTE?} To overcome the difficulty, we define a notion of approximation that does not rely on the estimate being additively close to the median but instead approximates the median by being close to in terms of quantiles, i.e., the rank in the sorted order of the values in $\mathbf{a-b}$.  Further, we show that we can approximately estimate MTE under this notion.
    \item \textbf{Limits of Approximating MTE.} We identify a quantity called \emph{variability}, that essentially captures the complexity of MTE estimation. We show that any algorithm that computes variability obtains the minimum approximation error, and also provides an efficient greedy algorithm for computing the variability.
\end{enumerate}
\end{mybox}

\subsection{Technical Overview}

Throughout the overview, we focus on binary outcomes (i.e., $k=2$), where $\mathbf{a}, \mathbf{b}$ are binary. All our results extend to the situation where we have $k$-ary outcomes. When $k=2$, the MTE which is the median of the difference vector $\mathbf{a - b}$ can be only one of the following values: $\{-1, 0, 1\}$. An experimenter provides us with an assignment while ensuring that every individual $i$ is assigned to either the treatment or control group but not both. Hence, we observe either $a_i$ or $b_i$ but not both. Further, the design may be adaptive or non-adaptive, i.e., the experimenter can choose to assign an individual based on all the previously assigned individuals and their responses.
We also use $[n]$ to denote the set $\{1,2,\ldots,n\}$ and for a distribution $\eta$ we use $\eta^{\otimes n}$ to denote the $n$-fold product distribution.

\subsubsection{Is Median Estimable?}
\label{subsection:median_estimable}
Consider two distributions $\mu_1$: $\begin{bmatrix}
1/3 & 0 \\
1/3 & 1/3 
\end{bmatrix} \mbox{ and } \mu_2: \begin{bmatrix}
0 & 1/3 \\
2/3 & 0
\end{bmatrix}$, that encode the joint distributions of the pair $(\mathbf{a}, \mathbf{b})$. So, essentially, we draw $n$ samples from $\mu_1$ or $\mu_2$ to construct $(\mathbf{a}, \mathbf{b})$. Here, the first row encodes the probabilities: For an $i$th individual drawn from $\mu_1$:  $\Pr_{\mu_1}[(a_i, b_i) = (0, 0)] = \mu_1[0, 0] = 1/3$, $\Pr_{\mu_1}[(a_i, b_i) = (0, 1)] =\mu_1[0, 1] =  0$, and so on.

Observe that for both $\mu_1$ and $\mu_2$, the marginal distributions are the same, i.e., $\Pr_{\mu_1}[a_i = 0] = \Pr_{\mu_2}[a_i=0] = 1/3$ and $\Pr_{\mu_1}[b_i = 0] = \Pr_{\mu_2}[b_i=0] = 2/3$. Therefore, the ATE of $\mathbf{a- b}$ for both the distributions is the same and equals $1/3$. However, the MTE for samples drawn from $\mu_1$ is $0$, and for $\mu_2$ is $1$. This is because the total fraction of $1s$ in a typical sample from $\mu_2$ is $2/3$ and hence contains the median, which is the middle element in the sorted order.  However, in $\mu_1$, the fraction of $-1$s is $0$ and the fraction of $0$s is $\mu_1[0, 0] + \mu_1[1, 1] = 2/3$, which exceeds the $1/2$ threshold, therefore, the MTE for $\mu_1$ is $0$. This example illustrates the inherent difficulty in estimating the median from the observed responses in the population. We formalize the notion and provide a brief intuition.

We can think of any algorithm for estimating the MTE as a rooted and labeled decision tree of depth $n$. Each non-leaf node $u$, is labeled with a distribution $\gamma_u$ defined over $\{a,b\}\times[n]$, where $a$ indicates treatment and $b$ indicates control. For each element say $(c,i)$ in the support there are two edges out of $u$ labelled $(c,i,0)$ and $(c,i,1)$ respectively. 
We start from the root, and whenever we are at a non-leaf node $u$, our algorithm samples  $(c,i)$ according to $\gamma_u$ and assigns the $i$th individual to treatment if $c=a$ and control if $c=b$. Suppose we see a response value of $0$ in individual $i$, then, we follow the edge $(c, i, 0)$; if $i$th we observe a response value $1$, we follow the edge $(c, i, 1)$. Proceeding in this way we trace a path, say $P$, from the root to a leaf of the decision tree. 
The feasibility constraint is that along any path $P$ that has a non-zero probability of being followed, any individual $i\in [n]$ is assigned exactly once. We define such 
decision trees as \emph{Randomized Feasible Decision Tree} (RFDTs) (see \Cref{defn:feasible_decision_tree,defn:randomized_fdt} for a precise definition).
At the end of following a path $P$, we output an estimate for the MTE say $\hat m\in \{-1,0,1\}$ (we can think of the leaf node of $P$ as being labeled with the output). 
Notice that our definition of RFDTs is general enough to capture non-adaptive designs. In fact, non-adaptive designs are those where the distribution $\gamma_u$ depends only on the depth of $u$ in the decision tree.

A crucial observation is that any RFDT, $R_n$ is incapable of distinguishing between two joint distributions over $\{0,1\}\times \{0,1\}$, say $\eta$ and $\tau$, which have the same marginals, i.e., if we sample $(a,b)\sim \eta$ and call the marginal of $a$ as $\eta_a$ and of $b$ as $\eta_b$ (similarly for $\tau_a$ and $\tau_b$), then, $\eta_a=\tau_a$ and $\eta_b=\tau_b$.
Specifically, any RFDT $R_n$ is unable to distinguish whether the input is drawn from the distribution $\eta^{\otimes n}$ or $\tau^{\otimes n}$. The reason is that for any path $P$ of $R_n$ the probability that $P$ is followed on the input $(\mathbf{a,b})$ drawn from $\eta^{\otimes n}$ is a function of the distributions associated with the nodes of $P$, and the marginal distributions $\eta_a\text{ and }\eta_b$. Since, this probability only depends on the marginals $\eta_a = \tau_a$ and $\eta_b=\tau_b$, hence the probability of following $P$ on $\tau^{\otimes n}$ remains the same (see \Cref{lem:indistinguishability_of_joints} for a detailed proof)

\subsubsection{Approximating the Median}\label{subsection:median_approximation}

As we saw previously, for two distributions $\eta$ and $\tau$, any algorithm (RFDT) is incapable of differentiating between $\eta^{\otimes n}$ and $\tau^{\otimes n}$.
We can only (approximately) estimate the marginals from the responses, and thus any algorithm that estimates MTE, must in principle use only the marginals. However, the MTEs of treatment and control vectors drawn from $\eta^{\otimes n}$ and $\tau^{\otimes n}$ can be quite different, as shown in the example taken previously.
Therefore, we realize that an unbiased estimation of the MTE is impossible. 
Notice the contrast with ATE estimation, as the ATE value is the same whether the data is generated according to $\eta^{\otimes n}$ and $\tau^{\otimes n}$. This is because the ATE is linear in the outcomes $\mathbf{a}, \mathbf{b}$ but the MTE is not.

Faced with the above roadblock in estimating the MTE, a natural question to ask if we can obtain an additive approximation to the MTE, i.e., output an $m\in \R$ such that the MTE is within $\pm\epsilon$ of the MTE for an approximation parameter $\epsilon>0$. 
However, for the situation of $k=2$ (or other values of $k$) this is not a desirable notion as we already know that the MTE is one of $\{-1,0,1\}$, so either we need to set $\epsilon=1$, which is too large, or know the MTE exactly. 

To overcome this, we introduce a notion of approximation that approximates the median by being close to in terms of quantiles, i.e., the rank in the sorted order. More precisely, we define $\hat m$ to be an $\epsilon$-\emph{quantile-approximation} of the median if it lies in the $0.5\pm \epsilon$-quantile band around the median, denoted by $(\mathbf{a-b})^{0.5+\epsilon}_{0.5-\epsilon}$. An $\epsilon$-quantile band around the median is the set of $\epsilon \cdot n$ values of the vector $\mathbf{a-b}$ in its sorted order, above and below the median.

\subsubsection{Optimality of Median Approximation}\label{subsection:median_optimality}

To accommodate the above notion of approximation we add a parameter to the output of an algorithm (RFDT) for estimating the MTE. In addition to an estimate $\hat{m}$ of the MTE, we also output the \emph{width} parameter $\epsilon$, with our guess being that $\hat{m}\in (\mathbf{a-b})^{0.5+\epsilon}_{0.5-\epsilon}$. 

How should we measure the performance of any algorithm (RFDT), say $R_n$ ($n$ denotes the number of individuals), that approximately estimates MTE? We measure it via two important parameters: how often is $R_n$ correct, i.e., $m$ is indeed an $\epsilon$-quantile approximation to the MTE, and what is the expected width $R_n$ outputs.
More precisely, for some small $\delta>0$ we desire that for \emph{all inputs} $(\mathbf{a,b})$ (there is no randomness in the inputs) with probability at least $1-\delta$ over the internal randomness of $R_n$ we have $\Pr[m\in (\mathbf{a-b})^{0.5+\epsilon}_{0.5-\epsilon}]$ where $(m,\epsilon)$ denotes the random output of $R_n$ on the input $(\mathbf{a,b})$. 
The width of $R_n$ on the input $(\mathbf{a,b})$, denoted by $\epsilon(\mathbf{a,b})$ is $\E[\epsilon]$, where the expectation is over the randomness of the algorithm.

Naturally, we desire to construct an RFDT $R_n$ such that it has a small error ($\delta$) and a small width for any input $(\mathbf{a,b})$. But how small can we make the width?
Recall the input distributions $\mu_1, \mu_2$ discussed in ~\Cref{subsection:median_estimable}. We know that $R_n$ can't distinguish between $\mu_1^{\otimes n}$ and $\mu_2^{\otimes n}$. 
Therefore, if upon observing the responses, $R_n$ decides to output $\hat{m}=0$, then to be correct with high probability it must output the width $\epsilon$ as at least $1/6$: this is because for a typical sample $\mathbf{a,b}$ from $\mu_2^{\otimes n}$ we have that the rank of $0$ in the sorted version of $\mathbf{a-b}$ is $n/3$. 
Similarly, we observe that the minimum widths required to be output by $R_n$ for each of the estimates $-1, 0,\text{ and } 1$ are $1/2, 1/6,\text{ and } 1/6$ respectively. Hence, $(1/6)^{th}$ is in principle a lower bound on the width of $R_n$ for inputs generated from $\mu_1^{\otimes n}$, $\mu_2^{\otimes n}$, and in fact from any $\mu^{\otimes n}$ that shares the same marginals.

We extend and formalize this notion of the minimum width required for any pair of marginal distributions $\eta_a$ over treatment values and $\eta_b$ over control values. For instance, for the distribution $\mu_1$, $\eta_a$ was $(1/3,2/3)$ (denoting the probabilities of being $0$ and $1$) and $\eta_b$ was $(2/3,1/3)$.
For a particular estimate $r\in \{-1,0,1\}$ we analyze the maximum deviation of its rank from $n/2$ across all joint distributions with the same marginals $\eta_a,\eta_b$. We call this the \emph{width-of-$r$} for $\eta_a,\eta_b$, and the minimum width-of-$r$ as the \emph{minimum median width} of $\eta_a,\eta_b$. See \Cref{defn:median_width} for precise definitions.
We capture the above via our notion of \emph{variability} (see \Cref{defn:variability}) which is a function of only the marginals and the estimate. Informally, it measures how much the rank of a particular estimate $r$ can vary across joint distributions consistent with the given marginals.
Moreover, we show that the minimum median width is always upper bounded by $\frac{1}{2} \cdot \frac{(2k-3)}{(2k-1)}$ when $\mathbf{a}, \mathbf{b}$ are $k$-ary (see \Cref{lem:upper_bd_min_median_width}). For $k=2$, this value turns out to be $1/6$ and the distributions described above are thus also the hard instances. 
Hard instances that obtain the upper bound on the minimum width for every $k$ are also detailed in Section~\ref{sec:variability}. 

So far, we have shown that variability (or minimum median width) essentially serves as a limit to MTE estimation.  
We argue, perhaps surprisingly that, variability also lends a tight computational handle on the MTE in the following sense.
First, in Section~\ref{sec:algorithm}, we provide an algorithm that computes variability exactly in time $O(k)$. We employ a greedy strategy in trying to identify an appropriate joint distribution that explains the marginals and also allows us to compute the variability.
Second, we show the algorithm that outputs the minimum median width of the marginals estimated from the observed responses is instance optimal.
The algorithm is called \medianestimator (\Cref{alg:medianestimator}) and the precise statement of optimality is given in \Cref{thm:lower_bd_width}.

\subsection{Our Contributions}
In this section, we highlight our key results. For the sake of brevity, we assume the following notation for all the below-stated results. Let  $\POk$ denote the set of $k$-ary potential outcomes, $\eta$ be any joint distribution on $\POk\times\POk$ with marginals $\eta_a$ and $\eta_b$ for the treatment and control respectively, and $\epsilon^*(\eta_a, \eta_b)$ denote the minimum median width (see \Cref{defn:median_width}).

In the \emph{instance-optimality} result described below, we show that any algorithm, encoded using an RFDT $R_n$, has an expected error that is at least the minimum median width $\epsilon^*(\eta_a, \eta_b)$. In other words, $\epsilon^*(\eta_a, \eta_b)$ essentially captures the minimum possible approximation error of any algorithm.
\begin{theorem}[Lower Bound on Width of any RFDT (Informal)] 
\label{thm:lbd_widht_informal}
     Let $R_n$ be an RFDT with width $\epsilon^{R_n}(\mathbf{a},\mathbf{b})$ where $(\mathbf{a,b})$ are the treatment and control vectors. If  for all $\mathbf{(a,b)}\in \POk^n \times \POk^n$, with high probability $R_n$ outputs an median estimate $\hat{m}$ and width $\epsilon$ such that $\hat{m}\in (\mathbf{a-b})^{0.5+\epsilon}_{0.5-\epsilon}$, then, we have:
    \[
    \E\limits_{(\mathbf{a},\mathbf{b})\leftarrow \eta^{\otimes n}}[\epsilon^{R_n}(\mathbf{a},\mathbf{b})] \gtrsim  \epsilon^{*}(\eta_a, \eta_b)
    \]
\end{theorem}

 Let $\psi_k = \frac{1}{2}\cdot \frac{2k-3}{2k-1}$. 
 The theorem below states that the minimum median width $\epsilon^*(\eta_a, \eta_b)$ is always upper bounded by $\psi_k$ and also shows a hardness result by identifying marginals $\bar{\eta}_a, \bar{\eta}_b$ that satisfy the upper bound.   
\begin{theorem}[Tight Bounds for Minimum Median Width (Informal)]For any $\eta_a, \eta_b$, we have: $\epsilon^*(\eta_a,\eta_b)$  $\le \psi_k$. And $\exists \bar{\eta}_a, \bar{\eta}_b$, such that $\epsilon^*(\bar{\eta}_a, \bar{\eta}_b) \geq \psi_k$, which implies $\epsilon^*(\bar{\eta}_a, \bar{\eta}_b) = \psi_k$.
Hence, when $k=2$, by \Cref{thm:lbd_widht_informal} and the above, $(1/6)^{th}$ is the fundamental limit on approximating the MTE in the quantile sense.
\end{theorem}

Finally, given observed potential outcome vectors (obtained using a Bernoulli design) $\hat{\mathbf{a}}, \hat{\mathbf{b}}$,  we provide a highly efficient greedy algorithm with linear running time, that returns a median estimate with a width roughly equal to the minimum median width of the empirical frequency distribution.
In light of \Cref{thm:lbd_widht_informal}, we know that essentially this is optimal.   
\begin{theorem}[Median Estimator Algorithm (Informal)]
Consider the \emph{Bernoulli Design} for assignment, i.e., each individual is assigned to control/treatment groups independently with probability $1/2$. 
Fix arbitrary treatment and control vectors $\mathbf{a,b}$ and let the observed response vectors be $\hat{\mathbf{a}}, \hat{\mathbf{b}}$. Then, there is an algorithm that returns a median estimator $\hat m$, with width $\epsilon \approx \epsilon^*(\eta_a,\eta_b) $ and runs in $O(n+k^2)$ time. Further, with high probability $\hat{m}\in (\mathbf{a-b})^{0.5+\epsilon}_{0.5-\epsilon}$.
\end{theorem}


\section{Preliminaries and Problem Formulation}

We use $a$ (and $b$) to denote the treatment (and control) value for a particular individual and $\mathbf{a} \ (\text{and } \mathbf{b})$ to denote the treatment (and control) values of all the $n$ individuals. As we do not observe all the treatment and control response values for every individual, we denote the partially observed responses using $\hat{\mathbf{a}}$ and $\hat{\mathbf{b}}$. Let $[n] = \rset{1,2,\ldots,n}$ and  let $\forall x \in \mathbf{R}$,  $(x)^{+}  = \max \{ x, 0 \}$. For a distribution $\eta$ we use $\eta^{\otimes n}$ to denote the $n$-fold product distribution.

\begin{definition}[$k$-ary outcomes]
Let $\POk \eqtriangle \{0, 1, \cdots, k-1\}$ be the set of all $k$-ary outcomes for both treatment and control groups, i.e., $\mathbf{a}$ and $\mathbf{b}$. Each potential outcome value for an individual $i \in [n]$ lies in the set $\POk$, i.e., $a_i, b_i \in \POk \quad \forall i \in [n]$. Therefore, $\ITE{i} \in \{-(k-1), \cdots, k-1\} \ \forall i \in [n]$.
\end{definition}

We make the following assumptions that are common in causality literature. 

\begin{assumption}[Independence]
 The experimental design is independent of the potential outcomes ${\mathbf{a}},{\mathbf{b}}$.
\end{assumption}

\begin{assumption}[SUTVA] The treatment assignment of a single individual does not affect the outcomes of other units.
\end{assumption}

Now, we begin with the definition of a quantile of any vector.

\begin{definition}[Quantiles of a vector]
\label{defn:upper_lower_qunatiles}
Let $\mathbf{v} \in \R^n$ be a vector, and let $\ell, u \in [0,1]$ be numbers such that $\ell\leq u$. For any $r \in \R$ define $q_{\ell}(r,\mathbf{v}) \coloneqq \frac{\abs{\rset{i\in [n]\mid \mathbf{v}_i < r}}}{n}$ and $q_{u}(r,\mathbf{v}) \coloneqq \frac{\abs{\rset{i\in [n]\mid \mathbf{v}_i \leq r}}}{n}$. Finally, define the set of \emph{$\ell$-to-$u$ quantiles of $\mathbf{v}$} as:
\[ \uplowquant{\mathbf{v}}{\ell}{u} \coloneqq \rset{r\in \R \mid [\ell, u] \cap [q_\ell(r,\mathbf{v}),q_u(r,\mathbf{v})] \neq \emptyset}.\]
\end{definition}

Next, we describe the concept of variability. Before that, we introduce some helpful related notions.
Let $\eta$ be a joint distribution over $\POk\times \POk$. We interpret it as a joint distribution  on $a, b$, i.e., 
\[\eta_{xy} \eqtriangle Pr[a = x, b = y]. \]

The marginal distribution denoted by $\eta_a$ is defined as:
$\forall x \in \POk: \eta_a (x) \eqtriangle{} \sum\limits_{y \in \POk} \eta_{xy}$. 
$\eta_b$ is defined similarly.
Let $\Delta_k$ be the simplex of all probability distributions over $\POk$:

   \[
\Delta_k \eqtriangle{} \left\{(g_0, g_1, \ldots, g_{k-1}) \mid g_i \geq 0 \text{ for all } i, \text{ and } \sum_{i=0}^{k-1} g_i = 1\right\}.
\]

Given marginal distributions $\eta_a$ and $\eta_b$, we define
 $\mathcal{J}(\eta_a, \eta_b)$ as the set of joint distributions $\eta$ with marginals $\eta_a$ and $\eta_b$ for the potential outcomes.
 
\begin{definition}[Lower Quantile]
Given a joint distribution $\eta$, the lower quantile of a particular value $r \in  \{-(k-1), \cdots, k-1\}$ in the range of ITE, is defined as the total probability mass of all possible $(x, y) \in \POk \times \POk$ that have an ITE value strictly less than $r$. Formally, we have:
 \[q_\ell(r,\eta) \eqtriangle{} \sum_{x \in \POk, y \in \POk} \mathbf{1} \{x - y < r\} \eta_{xy}.\]
\end{definition}

\begin{definition}[Upper Quantile]
Given a joint distribution $\eta$, the upper quantile of a particular value $r \in  \{-(k-1), \cdots, k-1\}$ in the range of ITE, is defined as the total probability mass of all possible $(x, y) \in \POk \times \POk$ that have an ITE value of at most $r$. Formally, we have:
 \[q_u(r,\eta) \eqtriangle{} \sum_{x \in \POk, y \in \POk} \mathbf{1} \{x - y \le r\} \eta_{xy}.\]
\end{definition}

Next, we prove a proposition that is useful when we want to translate between the lower quantile of a vector and a lower quantile of a distribution.
It roughly states that for a typical sample $(\mathbf{a,b})$ from $\eta^{\otimes n}$ (where $\eta$ is a joint distribution over $\POk\times \POk$) we have $q_\ell(r,\eta) = q_\ell(r,\mathbf{a}-\mathbf{b})$ and $q_u(r,\eta) = q_u(r,\mathbf{a}-\mathbf{b})$. Further, $(\mathbf{a,b}) \in \POk^n\times \POk^n$ there is a joint distribution $\eta$ such that $(\mathbf{a,b})$ is typical for $\eta^{\otimes n}$. 
\begin{proposition}
\label{prop:joint_dist_equiv_of_a,b}
     Let $\eta$ be a joint distribution over $\POk\times\POk$. A sample $(\mathbf{a,b})$ from $\eta^{\otimes n}$ is said to be typical if for all $(x,y)\in \POk\times \POk$ we have that the empirical frequency of $\rho(x,y) = \frac{|i\in [n]: (a_i,b_i)= (x,y)|}{n}$ is $\eta_{xy}$. 
     Suppose  $\mathbf{a,b}$ is a typical sample from $\eta^{\otimes n}$, then, for all $r\in \R$ we have $q_\ell(r,\eta) = q_\ell(r,\mathbf{a}-\mathbf{b})$ and $q_u(r,\eta) = q_u(r,\mathbf{a}-\mathbf{b})$.
     Further, for any $(\mathbf{a,b}) \in \POk^n\times \POk^n$ there is a joint distribution $\eta$ such that $(\mathbf{a,b})$ is typical for $\eta^{\otimes n}$.
\end{proposition}
\begin{proof}
    Suppose  $\mathbf{a,b}$ is a typical sample from $\eta^{\otimes n}$. Let $\mathbf{v = a-b}$. Then, 
    \begin{align*}
        q_{\ell}(r,\mathbf{v}) = \frac{\abs{\rset{i\in [n]\mid \mathbf{v}_i < r}}}{n} 
        = \sum_{x \in \POk, y \in \POk} \mathbf{1} \{x - y < r\} \rho_{xy} 
        & = \sum_{x \in \POk, y \in \POk} \mathbf{1} \{x - y < r\} \eta_{xy} \\
        & = q_{\ell}(r,\eta).
    \end{align*}
    A similar argument holds for $q_u(r,\eta) = q_u(r,\mathbf{a}-\mathbf{b})$. Next, given any $(\mathbf{a,b}) \in \POk^n\times \POk^n$ we associate a joint distribution $\eta$ over $\POk\times\POk$ as follows: sample $i\in[n]$ uniformly at random and output $(a_i,b_i)$.
    It is easy to see that $(\mathbf{a,b})$ is typical for $\eta^{\otimes n}$. 
\end{proof}

Next, we define the concept of variability which captures how much the quantiles of a given estimate $r$ can vary across all joint distributions that explain the given marginals.
\begin{definition}[Variability]
\label{defn:variability}  
Variability is a function $\nu: \mathbb{R} \times \Delta_k \times \Delta_k \rightarrow[0, 1] \times [0, 1]$, that takes as input an estimate $r$ and marginal distributions for the potential outcomes $\eta_a, \eta_b$, and returns a tuple as follows:
\[\nu(r, \eta_a, \eta_b) \eqtriangle{} (\nu_l(r, \eta_a, \eta_b) , \nu_u(r, \eta_a, \eta_b)), \text{ where:} \]
\begin{align*}
     \nu_l(r, \eta_a, \eta_b) &= \max_{\eta \in \mathcal{J}(\eta_a, \eta_b)} q_l(r,\eta)\\
     \nu_u(r, \eta_a, \eta_b) &= \min_{\eta \in \mathcal{J}(\eta_a, \eta_b)} q_u(r,\eta). 
\end{align*}
\end{definition}

Next, we understand that given the variability of an estimate $r$, how far can it be from the median in a quantile sense.

\begin{definition}[Minimum Median Width of Marginals]
\label{defn:median_width}
    Let $\eta_a$ and $\eta_b$ be distributions over $\POk$. Further, for any $r\in \R$ the \emph{width-of-$r$} for the pair $(\eta_a,\eta_b)$ is defined as 
\[ \varepsilon(r,\eta_a,\eta_b) := \max \{ \left( \nu_l(r, \eta_a, \eta_b) - 1/2\right)^+, (1/2 -  \nu_u(r, \eta_a, \eta_b))^+\}.\]

Then, the \emph{minimum median width} of the pair $(\eta_a,\eta_b)$, denoted by $\epsilon^*(\eta_a,\eta_b)$, is defined as  
\[\min_{r \in \{-(k-1), \cdots, (k-1)\} }\varepsilon(r,\eta_a,\eta_b).\]
\end{definition}
\begin{remark}
    Notice that for any sample $(\mathbf{a},\mathbf{b})$ from $\eta^{\otimes n}$ such that the marginals of $\eta$ are $\eta_a$ and $\eta_b$ we have that $r \in (\mathbf{a-b})^{0.5+\epsilon}_{0.5-\epsilon}$ where $\epsilon = \varepsilon(r,\eta_a,\eta_b)$.
\end{remark}

\subsection{Formalizing the Algorithm}

Next, we formalize the meaning of an algorithm that respects the inference constraints of observing either $a_i$ or $b_i$ but not both. At the end of observing the vectors $a$ and $b$ in a random manner (even in an adaptive manner), it outputs a member $m$ and a parameter $\epsilon \in [0,0.5]$ and the output is correct if $m \in \uplowquant{\mathbf{a-b}}{0.5-\epsilon}{0.5+\epsilon} $.

We first formalize the case of a deterministic algorithm, viewed as a decision tree. 

\begin{definition}[Feasible Decision Tree (FDT)]
\label{defn:feasible_decision_tree}
A \emph{feasible decision tree} (FDT) for block-length $n$ is a rooted and labeled $k$-ary tree of depth $n$. 
Each non-leaf node is labeled with some element from the set $\rset{a,b}\times [n]$, and each leaf node is labeled with an element from $\R \times [0,0.5]$. The edges of the tree are labeled with elements from $\POk$ such that from any non-leaf node the set of labels on the edges to its children is $\POk$. Finally, the tree should also satisfy the following \emph{feasibility constraint}. For any path $P$ from the root to some leaf and for all $i \in [n]$:

\[ \abs{\rset{\lab(u)\mid u\in P\text{ and is a non-leaf node} } \cap (\rset{a,b} \times \rset{i})} =1.\]
Here, $\lab(u)$ refers to the label of a vertex $u$.
\end{definition}

\begin{definition}[Output of an FDT]
Let $A_n$ be an FDT for block-length $n$.
On input $(\mathbf{a,b})\in \POk^n \times \POk^n$ the output of $A_n$, denoted as $A_n(\mathbf{a,b})$, is the label of the unique leaf obtained by following the path from the root to the leaf, where at each non-leaf node $u$, labeled $(c,i)$, we follow the edge from $u$ labeled with the same value as $c_i$: for instance, if $c=a$ and $a_i=0$, then we follow the edge from $u$ which is labeled $0$. So, the output is a tuple $(m,\epsilon) \in \R \times [0,0.5]$. 
\end{definition}

Next, we define the case of a randomized algorithm.

\begin{definition}[Randomized FDT (RFDT)]
\label{defn:randomized_fdt}
A randomized feasible decision tree (RFDT) for block-length $n$ is a distribution over the set of all FDTs for block-length $n$.
\end{definition}

\begin{remark}
\Cref{defn:randomized_fdt} is equivalent to the description of RFDTs given in \cref{subsection:median_estimable}. 
To see this, we can imagine pre-sampling all the internal randomness as required in the description of an RFDT in \Cref{subsection:median_estimable}. In particular, for all non-leaf nodes $u$ of the RFDT, we fix a sample according to $\gamma_u$ before proceeding along the RFDT. Given the samples from $\gamma_u$ for each $u$, the RFDT is now actually an FDT. Hence, overall the RFDT is a distribution over FDTs.
\end{remark}

\begin{definition}[Output of an RFDT]
    Let $R_n$ be an RFDT for block-length $n$ and let $\mu_n$ be its distribution over the set of FDTs. 
    On input $(\mathbf{a,b})\in \POk^n \times \POk^n$ the output of $R_n$, denoted as $R_n(\mathbf{a,b})$, is the random variable obtained in the following manner. Sample an FDT $A_n$ according to $\mu_n$ and output $A_n(\mathbf{a,b})$.
\end{definition}

Next, we define the notions of width and error of an RFDT.

\begin{definition}[Error and Width of an RFDT]\label{defn:width_rfdt}
Fix $\delta \in [0,1]$. An RFDT,$R_n$, for block-length $n$, is said to be a median approximator with error $\delta$ or a $\delta$-error median approximator, if for all $(\mathbf{a},\mathbf{b})\in \POk^n\times\POk^n$:
\[
\Pr\limits_{(m,\epsilon)\leftarrow R_n(\mathbf{a,b})}\left[m \in \uplowquant{a-b}{0.5-\epsilon}{0.5+\epsilon}\right]\geq 1 - \delta.
\]
Further, the width of $R_n$ on input $(\mathbf{a},\mathbf{b})$, denoted as $\epsilon^{R_n}(\mathbf{a},\mathbf{b})$, is defined as $\E\limits_{(m,\epsilon)\leftarrow R_n(\mathbf{a,b})}[\epsilon]$.
\end{definition}

\section{Variability \& Minimum Median Width}\label{sec:variability}

We first prove \Cref{lem:indistinguishability_of_joints} which roughly states that any RFDT is incapable of distinguishing between two joint distributions over $\POk\times \POk$, say $\eta$ and $\tau$, which have the same marginals, i.e., $\eta_a=\tau_a$ and $\eta_b=\tau_b$.
Specifically, any RFDT $R_n$ is unable to distinguish whether the input is drawn from the distribution $\eta^{\otimes n}$ or $\tau^{\otimes n}$.

The proof proceeds by noting that any RFDT $R_n$ is a distribution over FDTs. Further, for any FDT $A_n$ and for any path $P$ in $A_n$ the set of inputs $(\mathbf{a},\mathbf{b})$ where $A_n$ follows the path $P$, denoted by $F(A_n,P)$, is described by constraints on exactly one of $a_i$ or $b_i$ for $i\in [n]$. Hence, $\Pr_{(\mathbf{a},\mathbf{b})\leftarrow \eta^{\otimes n}}[(\mathbf{a},\mathbf{b}) \in F(A_n,P)]$ is only a function of the marginals $\eta_a,\eta_b$. Thus, the probability of $F(A_n,P)$ remains the same under $\tau$.

\begin{lemma}
\label{lem:indistinguishability_of_joints}
Let $\eta, \tau$ be joint distributions over $\POk \times \POk$ with same marginals for $a, b$, i.e., $\eta_a=\tau_a$ and $\eta_b=\tau_b$, and let $R_n$ be an RFDT for block-length $n$. 
Further, let $R_n(\eta), R_n(\tau)$ be the distributions of the output of $R_n$ when the input $(\mathbf{a,b})$ is sampled according to $\eta^{\otimes n}$ and $\tau^{\otimes n}$ respectively. Then, $R_n(\eta)=R_n(\tau)$. More precisely, for every $(m',\epsilon')\in \R \times [0,0.5]$ we have 
\[
\Pr\limits_{\substack{(\mathbf{a},\mathbf{b})\leftarrow \eta^{\otimes n}\\(m,\epsilon)\leftarrow R_n(a,b)}}[(m,\epsilon)=(m',\epsilon')]  = \Pr\limits_{\substack{(\mathbf{a},\mathbf{b})\leftarrow \tau^{\otimes n}\\(m,\epsilon)\leftarrow R_n(a,b)}}[(m,\epsilon)=(m',\epsilon')].
\]
\end{lemma}
\begin{proof}
    Fix some $(m',\epsilon')$. Let $\supp(R_n)$ denote the support of the distribution of $R_n$, i.e., the support of the distribution over FDTs of block-length $n$ which represents $R_n$. 
    Fix any FDT $A_n\in \supp(R_n)$ and any path from the root to leaf in $A_n$, say $P$, such that the label of the leaf node is $(m',\epsilon')$. 
    Further, let $F(A_n,P)$ denote those inputs, $(\mathbf{a},\mathbf{b}) \in\POk^n\times \POk^n$ such that $A_n$ follows the path $P$ on the input $(\mathbf{a},\mathbf{b})$.
    We will show that 
    \[
\Pr\limits_{\mathbf{a},\mathbf{b}\leftarrow \eta^{\otimes n}}[F(A_n,P)] =     \Pr\limits_{\mathbf{a},\mathbf{b}\leftarrow \tau^{\otimes n}}[F(A_n,P)].
    \]
    To see this, note that there is a function $f:[n]\to \POk$ and a subset $S\subseteq [n]$ such that $F(A_n,P) = \rset{(\mathbf{a},\mathbf{b})\mid a_i=f(i) \text{ for } i\in S, b_i = f(i) \text{ for } i \notin S}$. This follows from the feasibility constraint on $A_n$ as mentioned in \Cref{defn:feasible_decision_tree}.
    Hence, we have 
    \begin{align*}
      \Pr\limits_{\mathbf{a},\mathbf{b}\leftarrow \eta^{\otimes n}}[F(A_n,P)] &= 
      \prod_{i\in S}\eta_a(f(i)) \times \prod_{i \notin S} \eta_b(f(i))\\
      &= \prod_{i\in S}\tau_a(f(i)) \times \prod_{i \notin S} \tau_b(f(i))=\Pr\limits_{\mathbf{a},\mathbf{b}\leftarrow \tau^{\otimes n}}[F(A_n,P)].
    \end{align*}
    Taking a union over all paths $P$ in $A_n$ whose leaves are labelled with $(m',\epsilon')$ it follows that
        \[
    \Pr\limits_{\mathbf{a},\mathbf{b}\leftarrow \eta^{\otimes n}}[A_n((\mathbf{a},\mathbf{b}))=(m',\epsilon')] =     
     \Pr\limits_{\mathbf{a},\mathbf{b}\leftarrow \tau^{\otimes n}}[A_n((\mathbf{a},\mathbf{b}))=(m',\epsilon')].
    \]
    Since the above statement is true for every fixed $A_n$ in the support of $R_n$ the statement of the lemma follows.
\end{proof}

Using \Cref{lem:indistinguishability_of_joints} we prove the theorem below which serves to capture the hardness of approximation the MTE in terms of the variability of the marginal distributions of the treatment and control vector. As a consequence, if we can devise an algorithm for computing $\epsilon^*(\eta_a, \eta_b)$ exactly (section~\Cref{sec:algorithm}), then, then our algorithm is essentially optimal.

The proof idea is as follows. Fix a $\delta$-error median approximator $R_n$. Given a joint distribution $\eta$ which has marginals $\eta_a$ and $\eta_b$, for an estimate $r\in \{-(k-1),\ldots,k-1\}$, let $\gamma_r$ denote the probability that $R_n$ outputs $m=r$ and $\epsilon<\epsilon(r,\eta_a,\eta_b)$.
Now, consider the joint distribution $\eta^{(r)}$ which witnesses the width-of-$r$ value, i.e.,such that at least one of $\epsilon(r) = \left( q_\ell(r, \eta^{(r)}) - 1/2\right)^+$ or $\epsilon(r) =(1/2 -  q_u(r, \eta^{(r)}))^+$ is true. For $\eta^{(r)}$ the probability that $R_n$ outputs $m=r$ and $\epsilon<\epsilon(r,\eta_a,\eta_b)$ is still $\gamma_r$. However if $\epsilon<\epsilon(r,\eta_a,\eta_b)$, then, for a typical sample from $\left(\eta^{(r)}\right)^{\otimes n}$ it will be the case that $r \notin (\mathbf{a-b})^{0.5+\epsilon}_{0.5-\epsilon}$, and this counts as an error. Hence, $\gamma_r$ is not too much larger than $\delta$. Overall, taking a union over all values for $r$, we get that except for around a $2k\delta$ probability the output of $R_n$ has to be at least $\epsilon^{*}(\eta_a,\eta_b)$.

\begin{theorem}[Lower Bound on Width of any RFDT]
\label{thm:lower_bd_width}
    Let $\eta$ be any joint distribution on $\POk\times\POk$ with marginals $\eta_a$ and $\eta_b$ for the treatment and control respectively. 
    Further, let $R_n$ be a $\delta$-error median approximator for block-length $n$. Then, for any $\beta>0$ we have
    \[
    \E\limits_{(\mathbf{a},\mathbf{b})\leftarrow \eta^{\otimes n}}[\epsilon^{R_n}(\mathbf{a},\mathbf{b})]\geq (\epsilon^{*}(\eta_a,\eta_b)-\beta)\cdot (1-2k\delta) + 2k\exp(-2\beta^2\cdot n). 
    \]
\end{theorem}
\begin{proof}
    For $r \in \rset{-(k-1),\ldots,k-1}$ recall that (see width-of-$r$ in \Cref{defn:median_width})
    \[ \varepsilon(r,\eta_a,\eta_b) = \max \{ \left( \nu_\ell(r, \eta_a, \eta_b) - 1/2\right)^+, (1/2 -  \nu_u(r, \eta_a, \eta_b))^+\},\]
    where $\nu_\ell$ and $\nu_u$ are as defined in \Cref{defn:variability}.
    For the remainder of the proof we will use $\epsilon(r)$ and $\epsilon^*$ to denote $\epsilon(r,\eta_a,\eta_b)$ and $\epsilon^*(\eta_a,\eta_b)$ respectively.
    Note that $\epsilon^* = \min_r \epsilon(r)$. 
    
    Further, let $\eta^{(r)}$ denote a joint distribution over $\POk\times \POk$ with marginals $\eta_a,\eta_b$ such that at least one of $\epsilon(r) = \left( q_\ell(r, \eta^{(r)}) - 1/2\right)^+$ or $\epsilon(r) =(1/2 -  q_u(r, \eta^{(r)}))^+$ is true. 
    Now, let 
    \[\gamma_r = \Pr\limits_{\substack{(\mathbf{a},\mathbf{b})\leftarrow \eta^{\otimes n}\\ (m,\epsilon)\leftarrow R_n(\mathbf{a},\mathbf{b})}}[m=r,\epsilon<\epsilon(r)-\beta].\] 
    Notice that by \Cref{lem:indistinguishability_of_joints} we have 
    \[\gamma_r = \Pr\limits_{\substack{(\mathbf{a},\mathbf{b})\leftarrow \left(\eta^{(r)}\right)^{\otimes n}\\ (m,\epsilon)\leftarrow R_n(\mathbf{a},\mathbf{b})}}[m=r,\epsilon<\epsilon(r)-\beta],\]
    i.e., the value of $\gamma_r$ remains unchanged when inputs are sampled from $\left(\eta^{(r)}\right)^{\otimes n}$ instead of $ \eta^{\otimes n}$.

    Now, for an input $(\mathbf{a},\mathbf{b})$ sampled according to $\left(\eta^{(r)}\right)^{\otimes n}$ we have with at least $1-\exp(-2\beta^2\cdot n)$ probability that $r\notin(\mathbf{a}-\mathbf{b})^{0.5+\epsilon(r)-\beta}_{0.5-\epsilon(r)+\beta}$. 
    This follows from the Chernoff bound. Suppose  $\epsilon(r) = \left( q_\ell(r, \eta^{(r)}) - 1/2\right)$, and note that with probability at least  $1-\exp(-2\beta^2\cdot n)$ the empirical frequency of $\ITE{}$'s less than $r$ will be at least $q_\ell(r, \eta^{(r)}) -\beta$, and hence $q_\ell(r,\mathbf{a-b})\geq q_\ell(r, \eta^{(r)}) -\beta$ with probability at least $1-\exp(-2\beta^2\cdot n)$.
    
    Notice that this yields $\gamma_r\leq \delta + \exp(-2\beta^2\cdot n)$ as
    \begin{align*}
        \gamma_r &= \Pr\limits_{\substack{(\mathbf{a},\mathbf{b})\leftarrow \left(\eta^{(r)}\right)^{\otimes n}\\ (m,\epsilon)\leftarrow R_n(\mathbf{a},\mathbf{b})}}[m=r,\epsilon<\epsilon(r)-\beta]\\
        &\leq \Pr\limits_{\substack{(\mathbf{a},\mathbf{b})\leftarrow \left(\eta^{(r)}\right)^{\otimes n}\\ (m,\epsilon)\leftarrow R_n(\mathbf{a},\mathbf{b})}}[m=r,\epsilon<\epsilon(r)-\beta, r\notin(\mathbf{a}-\mathbf{b})^{0.5+\epsilon(r)-\beta}_{0.5-\epsilon(r)+\beta}] + \exp(-2\beta^2\cdot n)\\
        &\leq \Pr\limits_{\substack{(\mathbf{a},\mathbf{b})\leftarrow \left(\eta^{(r)}\right)^{\otimes n}\\ (m,\epsilon)\leftarrow R_n(\mathbf{a},\mathbf{b})}}[m\notin(\mathbf{a}-\mathbf{b})^{0.5+\epsilon}_{0.5-\epsilon}] + \exp(-2\beta^2\cdot n)\leq \delta+\exp(-2\beta^2\cdot n).
    \end{align*}

    Thus, 
   \begin{align*}
    \E\limits_{(\mathbf{a},\mathbf{b})\leftarrow \eta^{\otimes n}}[\epsilon^{R_n}(\mathbf{a},\mathbf{b})]&\geq  (\epsilon^*-\beta)\times \Pr\limits_{\substack{(\mathbf{a},\mathbf{b})\leftarrow \eta^{\otimes n}\\ (m,\epsilon)\leftarrow R_n(\mathbf{a},\mathbf{b})}}[\epsilon\geq \epsilon^*-\beta],\\
    \intertext{ which by a union bound over all $r\in \ITE{}$ gives}\\
    &\geq (\epsilon^* -\beta)\times\left(1-\sum_{r\in \rset{-(k-1),\ldots,k-1}}\gamma_r\right)\\ 
    &\geq (\epsilon^{*}(\eta_a,\eta_b)-\beta)\cdot (1-2k\delta) + 2k\exp(-2\beta^2\cdot n).
    \end{align*}
\end{proof}

\subsection{Bounds on Variability}\label{subsec:variability_bounds}

In this section, we show that the value of $\epsilon^*(\eta_a, \eta_b)$ is upper bounded by $\frac{1}{2} \cdot \frac{(2k-3)}{(2k-1)}$. In doing so, we have essentially shown that any algorithm that computes the variability exactly obtains an error that is upper bounded by $\frac{1}{2} \cdot \frac{(2k-3)}{(2k-1)}$.  

\begin{lemma}[Upper Bound on Minimum Median Width]
\label{lem:upper_bd_min_median_width}
    Given two distributions $\eta_a$ and $\eta_b$ over $\POk$, we have $\epsilon^*(\eta_a, \eta_b) \le \frac{1}{2} \cdot \frac{(2k-3)}{(2k-1)}$.
\end{lemma}
\begin{proof}
  For contradiction, let's assume that $\exists \eta_a, \eta_b \text{ such that } \epsilon^*(\eta_a, \eta_b) > \frac{1}{2} \cdot \frac{(2k-3)}{(2k-1)}$. 
  For brevity let $\psi_k = \frac{1}{(2k-1)}$ and let $\epsilon(r,\eta_a,\eta_b)$, $\nu_\ell(r, \eta_a, \eta_b)$ and $\nu_u(r, \eta_a, \eta_b)$ be denoted as $\epsilon(r),\nu_\ell(r)$ and $\nu_u(r)$ respectively. 
  Therefore, we have:
  $$\varepsilon(r) > 1/2 - \psi_k \quad \forall r \in \{-(k-1), \cdots, (k-1)\},$$ 
  where from \Cref{defn:median_width}
  $$\varepsilon(r) = \max \{ \left( \nu_l(r) - 1/2\right)^+, (1/2 -  \nu_u(r))^+\}.$$

Notice that $\nu_\ell(-(k-1)) = 0$, thus, $\nu_u(-(k-1))< \psi_k$. However, $\nu_u(k-1)=1$, thus $\nu_\ell((k-1)) > 1- \psi_k$. 
Now, let $r\in\ITE{}$ be the smallest value such that $\nu_{\ell}(r)>1-\psi_k$, therefore, $\nu_{\ell}(r-1)\leq 1-\psi_k$ and hence $\nu_{u}(r-1)<\psi_k$. 
We show that the existence of such an $r$ is impossible by analyzing $\E_{X\leftarrow\eta_a}[X]- \E_{Y\leftarrow\eta_b}[Y]$. 

Notice that for any joint distribution $\eta$ such that its marginals are $\eta_a$ and $\eta_b$ and $q_u(r-1,\eta)<\psi_k$ we have
\begin{align*}
    \E_{X\leftarrow\eta_a}[X]- \E_{Y\leftarrow\eta_b}[Y] &=\E_{X,Y\leftarrow\eta} [X-Y]> -(k-1)\cdot \psi_k + (r)\cdot (1-\psi_k).
\end{align*}

Similarly for any joint distribution $\eta'$ such that its marginals are $\eta_a$ and $\eta_b$ and $q_\ell(r,\eta)>1-\psi_k$ we have
\begin{align*}
    \E_{X\leftarrow\eta_a}[X]- \E_{Y\leftarrow\eta_b}[Y] &=\E_{X,Y\leftarrow\eta'} [X-Y]< (r-1)\cdot(1- \psi_k) + (k-1)\cdot \psi_k.
\end{align*}

Hence, we obtain 
\begin{align*}
    -(k-1)\cdot \psi_k + (r)\cdot (1-\psi_k) &<(r-1)\cdot(1- \psi_k) + (k-1)\cdot \psi_k\\
    &\implies 1-\psi_k < 2(k-1)\psi_k\\
    &\implies \frac{2k-2}{2k-1} < \frac{2k-2}{2k-1}, 
\end{align*}
which is a contradiction. 
Therefore, the posited $r$ can't exist and this completes the proof.
\end{proof}

We show a tightness result by arguing that there exists $\eta_a, \eta_b$ that satisfies the stated upper bound.
\begin{theorem}[Extremal distributions for Minimum Median Width]
\label{thm:lower_bound_minimum_median_width}
For any $k\in \N$, $r\in\POk$, let $\eta_a,\eta_b$ be distributions over $\POk$ defined as:
\[
\eta_a(r)= \frac{1}{2k-1}\cdot (1+\mathbf{1}\{ r>0 \}) \mbox{ and }
\eta_b(r)= \frac{1}{2k-1}\cdot (1+\mathbf{1}\{r<k-1\}).
\]

Then, $\epsilon^*(\eta_a,\eta_b) \geq \frac{1}{2}\cdot \frac{2k-3}{2k-1}$. Combined with \Cref{lem:upper_bd_min_median_width} we have $\epsilon^*(\eta_a,\eta_b)  = \frac{1}{2}\cdot \frac{2k-3}{2k-1}$.
\end{theorem}
\begin{proof}
    We first show that $\nu_u(0,\eta_a,\eta_b)\leq \frac{1}{2k-1}$. For this consider the joint distribution $\eta$ described as follows: sample $a\sim \eta_a$ and let $b=(a-1)\mod k$. It is easy to see that $b$ is distributed as $\eta_b$. Further, $a-b$ is $-(k-1)$ with probability $\frac{1}{2k-1}$ and is $1$ otherwise. Thus, $q_u(0,\eta)= \frac{1}{2k-1}$, and hence, $\nu_u(0,\eta_a,\eta_b)\leq \frac{1}{2k-1}$.

    Next, we show that $\nu_\ell(1,\eta_a,\eta_b)\geq \frac{2k-2}{2k-1}$. For this consider the joint distribution $\eta$ described as follows: sample $a\sim \eta_a$ and if $a\neq k-1$ let $b=a$; if $a=k-1$ then let $b=0\text{ or }k-1$ with probability $0.5$.
    It is easy to see that $b$ is distributed as $\eta_b$. Further, $a-b$ is $(k-1)$ with probability $\frac{1}{2k-1}$ and is $0$ otherwise. Thus, $q_\ell(1,\eta)= \frac{2k-2}{2k-1}$, and hence, $\nu_\ell(1,\eta_a,\eta_b)\geq \frac{2k-2}{2k-1}$.

    Notice that both $\nu_u(r,\eta_a,\eta_b)$ and $\nu_\ell(r,\eta_a,\eta_b)$ are increasing functions of $r$. Hence, for all $r\leq 0$ we have $\nu_u(r,\eta_a,\eta_b)\leq \frac{1}{2k-1}$ and hence for all $r\leq 0$ we have $\epsilon(r,\eta_a,\eta_b)\geq 0.5\cdot \frac{2k-3}{2k-1}$. Similarly, for all $r\geq 1$ we have $\nu_\ell(r,\eta_a,\eta_b)\geq \frac{2k-2}{2k-1}$ and hence  for all $r\geq 1$ we have $\epsilon(r,\eta_a,\eta_b)\geq 0.5\cdot \frac{2k-3}{2k-1}$. This proves the theorem.
\end{proof}



\section{Algorithm}\label{sec:algorithm}

In this section, we describe the Algorithm~\medianestimator that takes as input the observed response vectors $\hat{\mathbf{a}}$, $\hat{\mathbf{b}}$, and returns an estimate $\hat m$ for the median of the vector $\mathbf{a - b}$, along with an error estimate $\epsilon$; with a guarantee that $\hat m$ is within an $\epsilon$ band around the median, i.e., $\hat m \in \uplowquant{\mathbf{a-b}}{0.5-\epsilon}{0.5+\epsilon}.$ 

First, in section~\ref{subsec:medianestimator}, we describe an algorithm~\medianestimator that uses the algorithm \variability and returns a median estimate $\hat m$ and its $\epsilon$ band. Next, in section~\ref{subsec:variability} we describe an algorithm~\variability that computes variability $\nu(r, \eta_a, \eta_b)$, given a feasible median value $r \in \{-(k-1), \cdots, (k-1)\}$, and marginals $\eta_a$, $\eta_b$. 

\subsection{Approximating the Median}\label{subsec:medianestimator}
In Algorithm~\ref{alg:medianestimator}, we first compute the empirical distributions $\rho_a, \rho_b$ that are close approximations to $\eta_a, \eta_b$ (up to Chernoff errors). Then, using Algorithms~\variabilitylowerquantile and \variabilityupperquantile, we compute the variability for every feasible median value in the set $\{-(k-1), \cdots, (k-1)\}$. The minimizer in the set outputs the median estimate $\hat m$ and the corresponding minimum median width (see defn:\ref{defn:median_width}) $\epsilon(\hat m, \rho_a, \rho_b)$ as the approximation error.

\begin{algorithm}[!h]
\caption{\medianestimator}
\label{alg:medianestimator}
\begin{algorithmic}[1]
\STATE \textbf{Input:} Treatment and control response vectors $\hat{\mathbf{a}}, \hat{\mathbf{b}}$ and slack parameter $\beta>0$ 
\STATE \textbf{Output:} Median estimate $\hat m$ and the width $\epsilon$
\STATE Compute the empirical distributions $\rho_a$, $\rho_b$ as follows: 
\begin{align*}
    \forall j \in \POk \quad \rho_a(j) &= \frac{|\{i | \hat{\mathbf{a}}_i = j \text{ for } i \in \{1, 2, \cdots, n\}\}|}{n/2}\\
    \forall j \in \POk \quad \rho_b(j) &= \frac{|\{i | \hat{\mathbf{b}}_i = j \text{ for } i \in \{1, 2, \cdots, n\}\}|}{n/2}
\end{align*}
\FOR{$r$ in $\{-(k-1), \ldots, (k-1)\}$}
\STATE Compute $\nu_\ell(r, \rho_a, \rho_b)$ and $\nu_u(r,\rho_a,\rho_b)$ using \Cref{alg:variability-lower-quantile,alg:variability-upper-quantile}
\STATE Compute $\varepsilon(r, \rho_a, \rho_b) := \max \left\{ \left( \nu_l(r, \rho_a, \rho_b) - \frac{1}{2}\right)^+, \left(\frac{1}{2} -  \nu_u(r, \rho_a, \rho_b)\right)^+\right\}$
\ENDFOR
\STATE $\hat m \leftarrow \argmin\limits_{r \in \{- (k-1), \ldots, (k-1)\}} \varepsilon(r, \rho_a, \rho_b)$; 
$\epsilon \leftarrow \varepsilon(\hat m, \rho_a, \rho_b) + 2k\beta$.
\STATE \textbf{return} {$\hat m, \epsilon$}.
\end{algorithmic}
\end{algorithm}

\begin{theorem}
Let $\mathbf{a,b}\in \POk^n$ be the treatment and control vectors. Define the true empirical frequency vectors as $\eta_a(j) = \frac{|\{i \in [n] | {\mathbf{a}}_i = j \}|}{n}$ and $\eta_b(j) = \frac{|\{i\in [n] | {\mathbf{b}}_i = j \}|}{n}$.
Suppose that the \emph{Bernoulli Design} is used to assign the individuals to the treatment and control groups, i.e., each individual is assigned to one of the groups independently with probability $1/2$. 
Let the observed response vectors be $\hat{\mathbf{a}}, \hat{\mathbf{b}}$. Then, Algorithm~\medianestimator is a $\delta$-error median estimator $\hat m$ (\Cref{defn:width_rfdt}), with width $\epsilon \le \epsilon^*(\eta_a,\eta_b) + 2k \beta$, where $\delta  = 2k \cdot \exp{\left( - 2\beta^2 n\right)}$ and runs in $O(n+k^2)$ time.
\end{theorem}
\begin{proof}
Consider the observed empirical distributions $\rho_a$, $\rho_b$ (Line 3 in Alg~\ref{alg:medianestimator}) obtained from the response vectors $\hat{\mathbf{a}}, \hat{\mathbf{b}}$. From Chernoff and union bound, we have the following: $$\max_{j \in \POk} \max\{ |\rho_a(j) - \eta_a(j)|, |\rho_b(j) - \eta_b(j)|\} \le \beta,$$
with probability at least $1-2k \cdot \exp{\left( - 2\beta^2 n\right)}$.

 Therefore, we focus on the case when the distributions $\rho_a$ (or $\rho_b$) and $\eta_a$ (or $\rho_b$) are separated by a point-wise distance of at most $\beta$; hence the total variation (TV distance between $\rho_a$ (or $\rho_b$) and $\eta_a$ (or $\rho_b$) is at most $(1/2)k\beta$. Our goal is to calculate the deviation in variability (and width) due to this distance.
 
 Consider the comparison of $\nu_\ell(r,\eta_a,\eta_b)$ and $\nu_\ell(r,\rho_a,\rho_b)$ for some $r$. Suppose $\eta$ is the optimizer of the LP in \Cref{LP:1} for  $\nu_\ell(r,\eta_a,\eta_b)$. 
 Next, we show that there is a joint distribution $\rho$ with marginals $\rho_a$ and $\rho_b$ which is within a TV distance $k\beta$ of $\eta$. 
 Let $\mathcal{J}_a$ be an optimal coupling between $\eta_a$ and $\rho_a$, i.e., a joint distribution with marginals $\eta_a$ and $\rho_a$ such  that $\Pr\limits_{X,X'\leftarrow\mathcal{J}_a}[X\neq X']\leq  (1/2)k\beta$. Similarly, define $\mathcal{J}_b$. Then, $\rho$ is described as follows: sample $(X,Y)\leftarrow \eta$ and generate $X'\leftarrow \mathcal{J}_a|X$ and $Y'\leftarrow\mathcal{J}_b|Y$ independently, and output $(X',Y')$. It is easy to see that the distribution of $(X',Y')$ is $\rho$ and that $\Pr[(X,Y)\neq (X',Y')]\leq k\beta$. Therefore, $|\rho-\eta|_{TV}\leq k\beta$. However, this gives that $\Pr\limits_{(X,Y)\leftarrow \eta}[X-Y<r]$ is $\Pr\limits_{(X,Y)\leftarrow \eta}[X-Y<r]\pm k\beta$. Together with a symmetric argument starting with the optimizer of $\nu_\ell(r,\rho_a,\rho_b)$, we obtain $\nu_\ell(r,\eta_a,\eta_b)$ and $\nu_\ell(r,\rho_a,\rho_b)\pm \beta$. 
 Hence, we also obtain $\epsilon^*(\eta_a,\eta_b)$ is $\epsilon^*(\rho_a,\rho_b)\pm k\beta$.

However, the above implies that the $\epsilon$ output by \Cref{alg:medianestimator} is at least $\epsilon^{*}(\eta_a,\eta_b)$ and at most $\epsilon^{*}(\eta_a,\eta_b)+2k\beta$.
Also, by Proposition \ref{prop:joint_dist_equiv_of_a,b} let $\eta$ be the joint distribution such that $\mathbf{(a,b)}$ is typical for $\eta$. Then, we know that for each $r\in \ITE{}$, $q_\ell(r,\mathbf{a-b}) = q_\ell(r,\eta)\leq \nu_\ell(r,\eta_a,\eta_b)\leq \nu_\ell(r,\rho_a,\rho_b)+k\beta$ and $q_u(r,\mathbf{a-b}) = q_u(r,\eta)\geq \nu_u(r,\eta_a,\eta_b)\geq \nu_u(r,\rho_a,\rho_b)-k\beta$.
In particular, this is true for $r=m$, the output of \Cref{alg:medianestimator}. 
Therefore, $m\in (\mathbf{a-b})^{0.5+\epsilon}_{0.5-\epsilon}$. This completes the proof of correctness of \Cref{alg:medianestimator}.

 For the running time analysis observe that we calculate variability values for $2k-1$ possible median values (Line 4 in Algorithm~\medianestimator) after computing the empirical distributions $\rho_a, \rho_b$. Therefore, our claim regarding the running time follows from Lemmas~\ref{lem:variability-lower-quantile} and~\ref{lem:variability-upper-quantile}.
\end{proof}

\subsection{Computing Variability}\label{subsec:variability}

In section~\ref{sec:variability}, we have shown that computing the minimum median width of marginals (by calculating the variability for every feasible median value) captures the inherent complexity of approximating the median. In this section, we show that there is a natural greedy algorithm that computes the variability exactly.  To compute variability, we need to compute $\nu_{l}(r, \eta_a, \eta_b)$ and $\nu_u(r, \eta_a, \eta_b)$, where $r$ is a feasible median value and $\eta_a, \eta_b$ are the marginals. As $\nu_{l}(r, \eta_a, \eta_b) = \max_{\eta} \sum_{x, y} \mathbf{1} \{ x - y < r \} \ \eta_{xy}$, we need to identify an appropriate joint distribution $\eta$, that maximizes the total probability for the lower quantile. It turns out that we can formulate this as a linear program shown below (see Figure~\ref{LP:1}), with marginal and non-negativity constraints.

\begin{figure}[!h]
    \begin{align*}
        \underline{\text{Objective}}: & \quad \max \sum_{x, y} \mathbf{1} \{x - y < r\} \eta_{xy}\\
        \text{(Marginal Constraints)}& \quad \sum_{y \in \POk} \eta_{xy} = \eta_a(x) \quad \forall x \in \POk\\
        \text{(Marginal Constraints)}& \quad \sum_{x \in \POk} \eta_{xy} = \eta_b(y) \quad \forall y \in \POk\\
        \text{(Non-negativity)}& \quad \eta_{xy} \ge 0 \quad \forall x, y \in \POk \times \POk
    \end{align*}
\caption{Linear Program for computing the lower quantile component of variability}
\label{LP:1}

\end{figure}

\begin{algorithm}[!h]
\caption{\variabilitylowerquantile}
\label{alg:variability-lower-quantile}
\begin{algorithmic}[1]
\STATE \textbf{Input:} Marginal distributions $\eta_a, \eta_b$ and feasible median value $r \in \{-(k-1), \cdots, (k-1) \}$.
\STATE \textbf{Output:} $\nu_l(r, \eta_a, \eta_b)$.
\STATE $\forall x \in \POk, y \in \POk$, set $\eta_{xy} \leftarrow 0$
\FOR{$y \in \{k-1, k-2, \cdots, 0\}$}
\FOR{$x \in \{y + r-1, y+ r-2, \cdots, 0\} \cap \POk$}
\STATE Identify all the marginal constraints in LP~\ref{LP:1} containing $\eta_{xy}$ and then slowly increase the value of $\eta_{xy}$ until one of them becomes tight. 
\STATE We mark the corresponding column or row with the tight constraint as \emph{frozen} and do not update any values in them.
\ENDFOR
\ENDFOR
\STATE \textbf{Return} $\nu_l(r, \eta_a, \eta_b) \leftarrow \sum_{x, y \in \POk \times \POk} \mathbf{1} \{ x - y < r \} \eta_{xy}$.
\end{algorithmic}
\end{algorithm}
We can view solving the linear program (in Figure~\ref{LP:1}) as essentially filling in the entries of a 2-dimensional matrix of size $\POk \times \POk$ such that the difference in indices is strictly smaller than $r$, i.e., $x - y < r \ \forall x, y$. To do so, we employ a greedy strategy in Algorithm~\variabilitylowerquantile (see Alg~\ref{alg:variability-lower-quantile}). Without loss of generality, we fill the matrix, greedily from the last column and last row and move up the column (see Lines 4-5 in Alg~\ref{alg:variability-lower-quantile}). Our greedy strategy involves increasing the value of $\eta_{xy}$ until one of the marginal constraints is tight and freezing the values of the row/column corresponding to the constraint (see Lines 7-10 in Alg~\ref{alg:variability-lower-quantile}). A similar approach results in an algorithm~\ref{alg:variability-upper-quantile} for computing the upper quantile component of variability. Formally:

\begin{lemma}\label{lem:variability-lower-quantile}
Given a feasible value $r \in \{-(k-1), \cdots, (k-1) \}$ and marginals $\eta_a, \eta_b$, there is an algorithm for computing $\nu_l(r, \eta_a, \eta_b)$ exactly, and runs in $O(k)$ time.    
\end{lemma}
\begin{proof}
         Let $\eta^*$ denote the output of the greedy Algorithm and $\Bar{\eta}$ be an optimizer of \autoref{LP:1}. Let us assume: \[\sum_{x, y} \mathbf{1} \{x - y < r\} \etastar_{xy} < \sum_{x, y} \mathbf{1} \{x - y < r\} \etabar_{xy}\]
         As the greedy algorithm fills the entries of $\etastar_{xy}$ by making sure that one of the constraints is tight: we will be able to identify a pair (wlog first pair) $x, y$ such that $\etastar_{xy} > \etabar_{xy}$. Otherwise, we would have $\etabar = \etastar$.
         Now consider the marginal constraint \[\sum_{x \in \POk} \etastar_{xy} = \sum_{x \in \POk} \etabar_{xy} =  \eta_b(y).\] Therefore, we have a $u$ such that: 
        $\etastar_{uy} < \etabar_{uy}$.
         Similarly the marginal constraint
        \[\sum_{y \in \POk} \etastar_{xy} = \sum_{y \in \POk} \etabar_{xy} =  \eta_a(x).\]
        Therefore, we have a $v$ such that:
        $\etastar_{xv} < \etabar_{xv}$.
         We update $\etabar$ by increasing the value of $\etabar_{xy}$ as follows:
        \begin{enumerate}
            \item $\etabarup_{uy} = \max\{ 0, \etabar_{uy} - (\etastar_{xy} - \etabar_{xy})\}$
            \item $\etabarup_{xv} = \max\{0, \etabar_{xv} - (\etastar_{xy} - \etabar_{xy})\}$
            \item $\etabarup_{uv} = \etabar_{uv} + (\etastar_{xy} - \etabar_{xy})$
            \item $\etabarup_{xy} = \etabar_{xy} + (\etastar_{xy} - \etabar_{xy})$
        \end{enumerate}
         Observe that such an update satisfies the marginal constraints and the objective function remains the same for both $\etabar$ and $\etabarup$.
         The updated $\etabarup$ such that $\etabarup_{xy} = \etastar_{xy}$. By repeatedly following this updating process, we will have $\etabarup = \etastar$, a contradiction to our claim. As we eliminate a row or column by making a marginal constraint tight in each iteration of our algorithm. The total number of iterations and the values of $(x, y) \in \POk \times \POk$ pairs that we fill are at most $2k$. Hence, the running time of our algorithm is $O(k)$. 
\end{proof}

As $\nu_{u}(r, \eta_a, \eta_b) = \min_{\eta} \sum_{x, y} \mathbf{1} \{ x - y \le r \} \ \eta_{xy}$, we need to identify an appropriate joint distribution $\eta$, that minimizes the total probability for the upper quantile. It turns out that we can formulate this as a linear program shown below (see Figure~\ref{LP:2}), with marginal and non-negativity constraints. Similar to Algorithm~\ref{alg:variability-lower-quantile}, we have Algorithm~\variabilityupperquantile that computes the value $\nu_{u}(r, \eta_a, \eta_b)$ exactly.

\begin{figure}[!h]
    \begin{align*}
        \underline{\text{Objective}}: & \quad \min \sum_{x, y} \mathbf{1} \{x - y \le r\} \eta_{xy}\\
         &\quad \sum_{y \in \POk} \eta_{xy} = \eta_a(x) \quad \forall x \in \POk\\
         &\quad \sum_{x \in \POk} \eta_{xy} = \eta_b(y) \quad \forall y \in \POk\\
         &\quad \eta_{xy} \ge 0 \quad \forall x, y \in \POk \times \POk
    \end{align*}
\caption{Linear Program for computing the upper quantile component of variability}
\label{LP:2}

\end{figure}

\begin{algorithm}[!h]
\caption{\variabilityupperquantile}
\label{alg:variability-upper-quantile}
\begin{algorithmic}[1]
\STATE \textbf{Input:} Marginal distributions $\eta_a, \eta_b$ and feasible median value $r \in \{-(k-1), \cdots, (k-1) \}$.
\STATE \textbf{Output:} $\nu_u(r, \eta_a, \eta_b)$.
\STATE $\forall x \in \POk, y \in \POk$, set $\eta_{xy} \leftarrow 0$
\FOR{$y \in \{k-1, k-2, \cdots, 0\}$}
\FOR{$x \in \{y + r, y+ r-1, \cdots, 0\}$}
\STATE Identify all the marginal constraints in LP~\ref{LP:2} containing $\eta_{xy}$ and then slowly increase the value of $\eta_{xy}$ until one of them becomes tight. 
\STATE We mark the corresponding column or row with the tight constraint as \emph{frozen} and do not update any values in them.
\ENDFOR
\ENDFOR
\STATE \textbf{Return} $\nu_l(r, \eta_a, \eta_b) \leftarrow \sum_{x, y \in \POk \times \POk} \mathbf{1} \{ x - y \le r \} \eta_{xy}$.
\end{algorithmic}
\end{algorithm}

\begin{lemma}\label{lem:variability-upper-quantile}
Given a feasible median value $r \in \{-(k-1), \cdots, (k-1) \}$ and marginals $\eta_a, \eta_b$, Algorithm~\ref{alg:variability-upper-quantile} computes the value $\nu_u(r, \eta_a, \eta_b)$ exactly and runs in time $O(k)$. 
\end{lemma}
\begin{proof}
    A similar proof as that of \cref{lem:variability-lower-quantile}.
\end{proof}



\section{Conclusion}
In this work, we studied the task of median treatment effect estimation, to capture the heterogeneity in the data. We argued that the task is inestimable, provided new notions of approximations, and showed optimal algorithms with minimum approximation error. We want to highlight that our results make no distributional assumptions, and can be extended for any quantile $q < 0.5$ (in addition to median), by modifying the definition of variability (\cref{defn:variability}) appropriately. For future work, a potential direction is to study the settings where the potential outcomes are continuous.


\acks{The authors would like to thank the following individuals for many helpful discussions: Anup Rao, Kirankumar Shiragur, Christopher Harshaw, Shiva Kasiviswanathan, and Avi Feller.
The authors also would like to thank the organizers for the invitation to the Causality program at the Simons Institute for the Theory of Computing, which helped facilitate early discussions of this work.}

\bibliographystyle{plainnat}
\bibliography{ref}
\pagebreak


\end{document}